\documentclass[11pt]{article}
\usepackage{amsthm, amsfonts, amssymb,latexsym,amsmath, color}
\usepackage[a4paper, total={6in, 8in}]{geometry}
\usepackage[dvipsnames,svgnames,table]{xcolor}
\usepackage[colorlinks=true,linkcolor=RoyalBlue,urlcolor=RoyalBlue,citecolor=PineGreen]{hyperref}
\usepackage{tikz}
\usepackage{tikz-cd} 

\usepackage{comment, caption}
\usepackage{graphicx,subcaption}

\usepackage{pgf,tikz,pgfplots}
\pgfplotsset{compat=1.14}
\usepackage{pgf,tikz,pgfplots}
\usepackage{mathrsfs}
\usetikzlibrary{arrows}
\usepackage{mathrsfs}
\usetikzlibrary{arrows}
\usepackage{capt-of}
\usetikzlibrary{decorations.pathreplacing}
\usetikzlibrary{cd}
\usetikzlibrary{positioning}
\usetikzlibrary{calc}
\usepackage{algorithmic}
\usepgfplotslibrary{fillbetween}
\usetikzlibrary{intersections}
\usetikzlibrary{patterns}

\usepackage[nameinlink]{cleveref}

\title{Positivity sets of hinge functions}
\author{Josef Schicho, Ayush Kumar Tewari, Audie Warren}

\newcommand{\R}{\mathbb{R}}

\newcommand{\conv}{\text{Conv} \, }

\newtheorem{prop}{Proposition}
\newtheorem{lemma}{Lemma}

\newtheorem{theorem}{Theorem}

\theoremstyle{case}

\theoremstyle{remark}

\newtheorem{example}{Example}

\begin{document}
 \maketitle
 \begin{abstract}
     In this paper we investigate which subsets of the real plane are realisable as the set of points on which a one-layer ReLU neural network takes a positive value. In the case of cones we give a full characterisation of such sets. Furthermore, we give a necessary condition for any subset of $\mathbb R^d$. We give various examples of such one-layer neural networks.
 \end{abstract}

\section{Introduction}

In this paper, a \textit{hinge function} is a representation of a one-layer ReLU neural network which allows skip connections; we will define a hinge function as a continuous piecewise affine linear function given by
\begin{gather}\label{def:hingemin}
    h:\mathbb R^d \rightarrow \mathbb R \\
    h(\boldsymbol{x}) = L_0(\boldsymbol{x}) + \sum_{i=1}^s \min(L_i(\boldsymbol{x}),0) - \sum_{j=s+1}^t \min(L_j(\boldsymbol{x}),0)
\end{gather}
where each $L_i$ is an affine linear function. The ReLU function is given by $x \mapsto \max(x,0)$; note that the minima involved may be swapped for maxima. The case of disallowing skip connections is given when $L_0(\boldsymbol{x})$ is a constant. In this paper we will use an equivalent formulation of hinge functions, as a continuous piecewise linear function given by 
\begin{gather}\label{def:hingemod}
    h(\boldsymbol{x}) = L_0(\boldsymbol{x}) + \sum_{i=1}^s |L_i(\boldsymbol{x})| - \sum_{j=s+1}^t |L_j(\boldsymbol{x})|.
\end{gather}
 The equivalence of the two definitions is immediate by using $\min(x,0) = \frac{|x|+x}{2}.$ In the case $d=2$, where we have most of our examples, each $L_i$ is of the form $ax + by + c$, where $\boldsymbol{x}=(x,y).$

Hinge functions are functions that can be represented by a single hidden layer neural network with ReLU activation function, as has been pointed out in \cite{Wang_Sun:05}. In that paper, the authors show that a continuous piecewise linear function which can be written as a linear combination of maxima with at most $n$ affine linear functions, can be represented by a ReLU neural network with the number of hidden layers bounded by $\lceil \log_2(n) \rceil$. The paper \cite{Hertrich:23} shows that function $(x,y)\mapsto \max(x,y,0)$ is not a hinge function; it requires two hidden layers. An algorithm that decides whether a given piecewise linear function is a hinge function is given in \cite{Koutschan:24}. Moreover, there are approximation theorems (e.g. \cite{Hornik:89}) stating that any continuous function can be approximated on a compact set by hinge functions. 

A prominent task of a neural network is {\em classification}. In the simplest case (one-class-classification), an input vector is accepted if and only if the value of the output function is positive. Since the ReLU function does not change positivity, an input vector is accepted if and only if the hinge function which is fed into the output neuron has positive value. In this paper, we investigate the question of which subsets of $\R^d$ may be accepted (classified) by a single layer neural network with ReLU activation. For a hinge function $h:\mathbb R^d \rightarrow \mathbb R$, the set of points where $h$ takes a positive value is called the positivity set of $h$. A subset $P \subseteq \mathbb R^d$ is called a \textit{positivity set} if it is the positivity set of some hinge function. The question of which subsets of $\mathbb R^d$ can be classified by a hinge function is then equivalent to the following question:
%
    which subsets $P \subseteq \mathbb R^d$ are positivity sets of a hinge function?

We can give various examples of sets which are indeed positivity sets, but a general characterisation may be difficult to obtain - we instead focus on a local case in dimension $2$, where $P$ is an open cone. Our main result is a characterisation for such positivity sets, and also a local necessary condition for a general point set to be a positivity set in $\mathbb R^d$. From here on, all cones mentioned are open cones. In order to state our results, we will need some preliminary definitions. 

\begin{figure}[ht]
    \centering
    \includegraphics[width=0.48\linewidth]{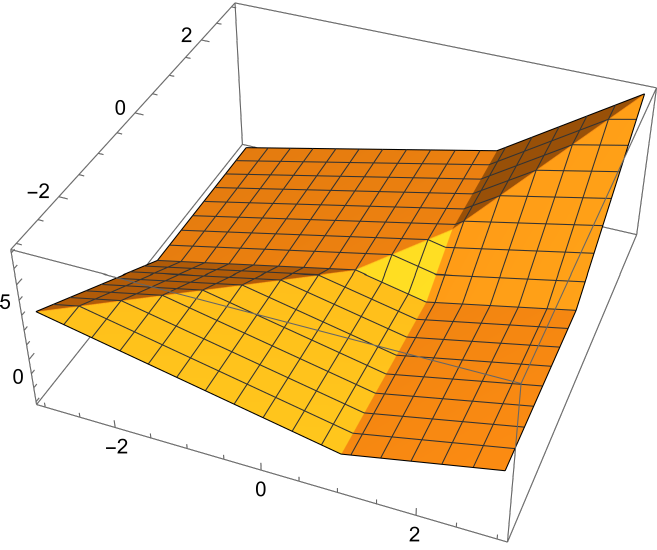}
    \hspace{15mm}
    \includegraphics[width=0.4\linewidth]{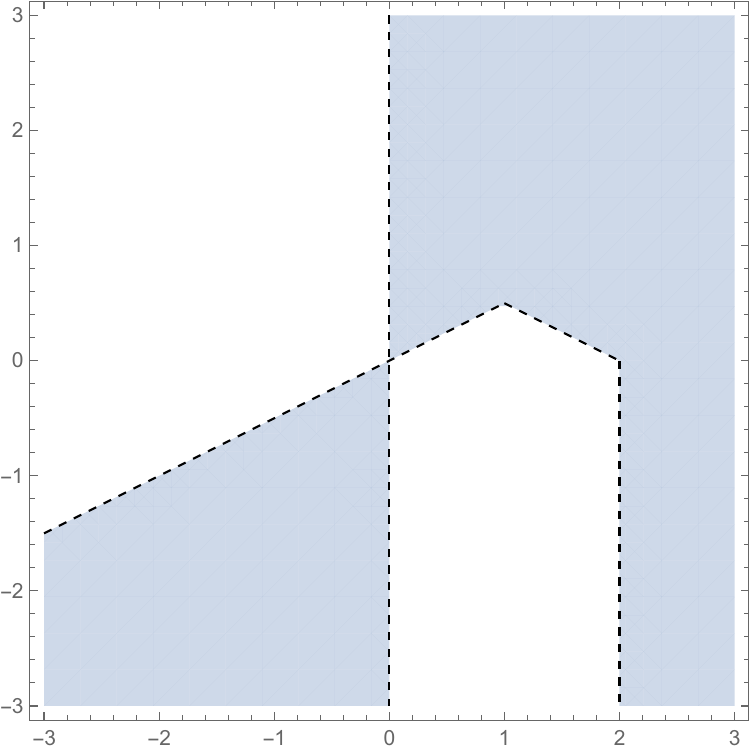}
    \caption{A hinge function and its corresponding positivity set. The hinge function is given by $$h(x,y) =x-1 +|x-1|+|y|-|y-x|.$$}
    \label{fig:hingeandpositive}
\end{figure}

\section{Preliminaries}

\subsection{Statement of main results}

Let $P$ be a subset of $\mathbb R^d$. If $P$ is a finite union of open polytopes (maybe unbounded), then we call it a \textit{polytopal set}. A polytopal set which is also a cone (that is, if it contains $p$, then it also contains $\lambda p$ for every $\lambda>0$), we will name a \textit{polytopal cone}. Note that any positivity set of a hinge function must be a polytopal set.

Let $C$ be a polytopal cone. Our results can be stated with reference to two sets which depend on $C$, given by the following definitions.
$$R_C := \{p \in C : -p \notin C\}, \quad G_C := \{p \in \partial C : -p \in \partial C\}.$$
In words, the set $R_C$ consists of those points of $C$ which are \textit{not} centrally symmetric. The set $G_C$ is itself centrally symmetric, and consists of any full lines through the origin which are contained in the boundary of $C$. From the next section on, we will drop the index $C$ since it will be clear from the context. 

We can now state our main result. In the following theorem, $\langle G_C \rangle$ refers to the linear span of $G_C$, and $\conv (R_C)$ is the convex hull of $R_C$.

\begin{theorem}\label{thm:main}
    Let $C \subseteq \mathbb R^2$ be a polytopal cone. Then $C$ is the positivity set of a hinge function if and only if
    $$\langle G_C \rangle \cap \conv (R_C) = \emptyset.$$
\end{theorem}

Our second result shows that for any polytopal set $P \subseteq \mathbb R^d$, there is a local condition given by the analogue of the statement in Theorem \ref{thm:main} which must be fulfilled locally, or the set cannot be a positivity set. To state this result, we need the notion of a \textit{local cone}. 

Let $P \subseteq \mathbb R^d$ be a polytopal set. For such a set, each point $q \in \mathbb R^d$ has a \textit{local cone} $C_q$ with respect to $P$. This local cone is defined in the following way: take $\epsilon >0$ small enough that the intersection of the $\epsilon$-ball $B_{\epsilon}(q)$ around $q$ with $P$ looks like a cone, that is, if $p \in P \cap B_{\epsilon}(q)$, then the whole line segment from $q$ through $p$ to the boundary of the ball, is contained within $P$. We then translate the entire set so that $q$ is moved to the origin, and then we make it a cone by adding all multiples by postive scalars. We name the resulting polytopal cone $C_q$.

\begin{figure}[ht]
    \centering
    \includegraphics[width=0.3\linewidth]{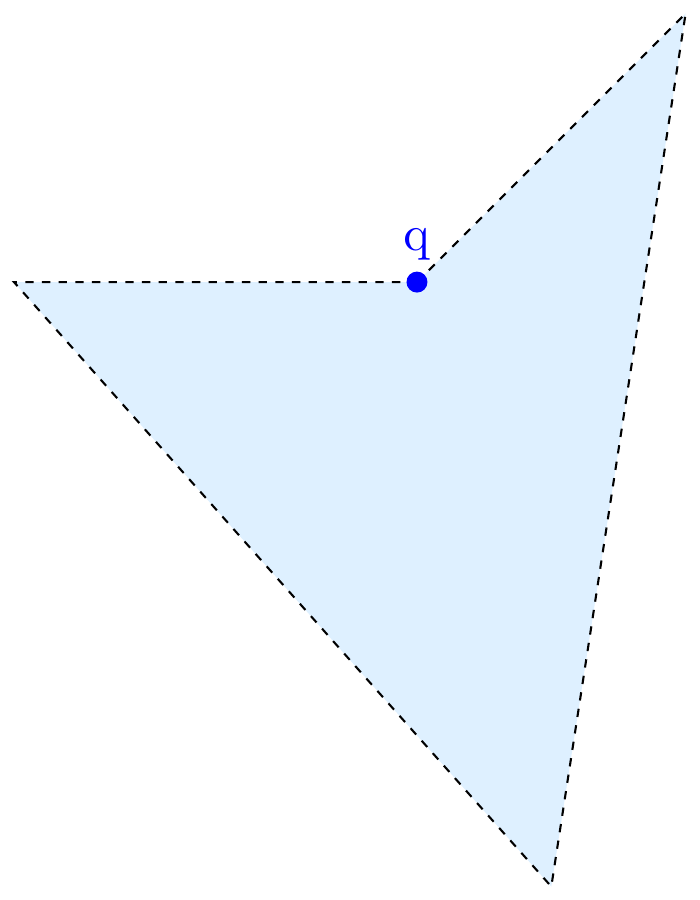}
    \hspace{30mm}
    \includegraphics[width=0.35\linewidth]{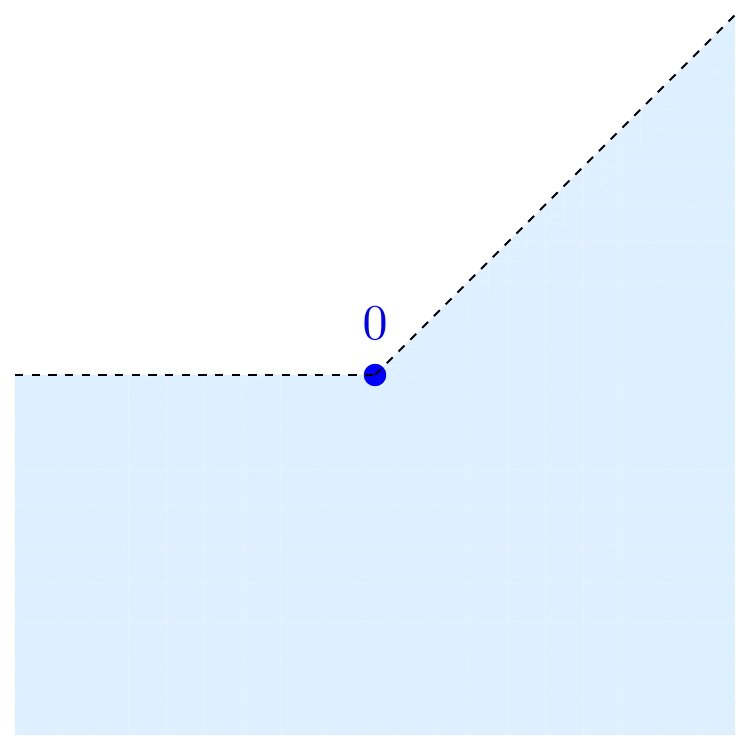}
    \caption{The point $q$ in the blue set $P$, and its local cone $C_q$.}
    \label{fig:localcone}
\end{figure}

We can now state our second result.

\begin{theorem}\label{thm:localcond}
    Let $P$ be a polytopal subset of $\mathbb R^d$. If there is a point $q \in \mathbb R^d$ such that for the cone $C_q$ we have $\langle G \rangle \cap \conv (R) \neq \emptyset$, then $P$ cannot be a positivity set.
\end{theorem}

Before proving Theorems \ref{thm:main} and \ref{thm:localcond}, we need some more definitions and preliminary results.

\subsection{Valleys, mountains, and local gradients}

Given a hinge function $h: \mathbb R^d \rightarrow \mathbb R$, $h(\boldsymbol{x}) = L_0(\boldsymbol{x}) + \sum_{i=1}^s |L_i(\boldsymbol{x})| - \sum_{j=s+1}^t |L_j(\boldsymbol{x})|$, the points of non-differentiability of $h$ within $\mathbb R^d$ correspond to the zero sets of the affine linear functions $L_i$. If such an $L_i$ appears with coefficient $1$ in $h$, we will call $Z(L_i) \subseteq \mathbb R^d$ a \textit{valley}. If instead it comes with coefficient $-1$, we call it a \textit{mountain}. These names come from visualising the graph of $h$ in the case $d=2$. Locally around a point $p$ not lying on a mountain or valley, $h$ is an affine linear function.

The mountains and valleys yield a tessellation of the input space, such that $h$ is affine linear in the interior of each tile. If we consider crossing from one tile to another via a single mountain or valley $Z(L_i)$, the local representation of $h$ changes by either adding or taking away $2L_i$. Each tile of the tessellation corresponds to a vector in $\{-1,1\}^{s+t}$, which is determined by which side of each $Z(L_i)$ the tile lies on (formally, whether $L_i$ is positive or negative in the tile). This vector is enough to give the local representation of $h$ within the tile.

\begin{figure}[ht]
    \centering
    \includegraphics[width=0.45\linewidth]{hingefunction1.pdf}
    \hspace{15mm}
    \includegraphics[width=0.4\linewidth]{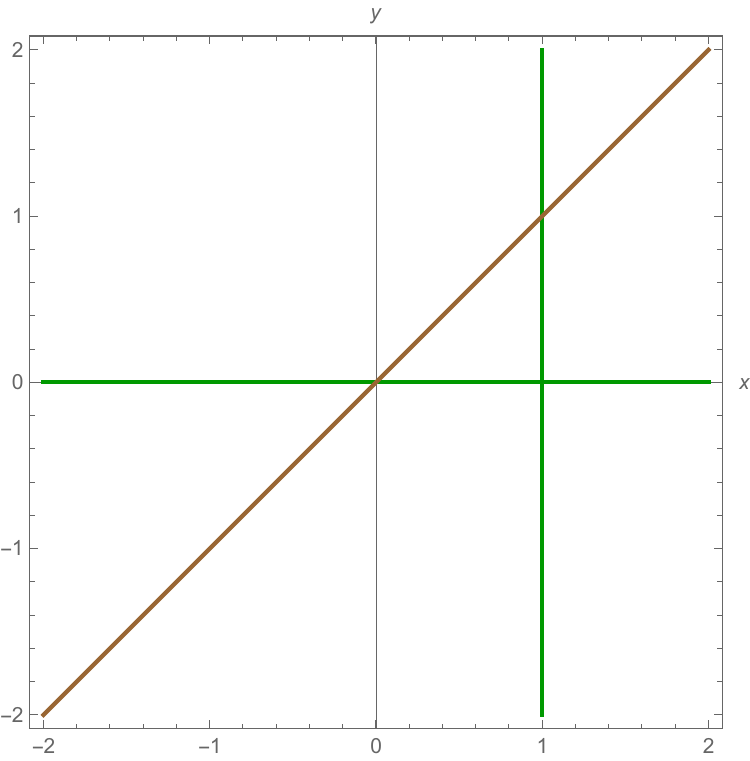}
    \caption{The hinge function $h$ from Figure \ref{fig:hingeandpositive} and its mountain/valley arrangement. The mountain is brown and the valleys are green.}
    \label{fig:enter-label}
\end{figure}

We begin with a useful lemma regarding the local behaviour of the gradient of hinge functions. We do not claim originality, and this result can be found (with different notations) in \cite{tao2022piecewise, chua1988canonical}.

\begin{lemma}\label{lem:localgradient}
    Let $h: \mathbb R^d \rightarrow \mathbb R$ be a hinge function, and let $p \in \mathbb R^d$. Let $U_p$ be an open ball centered at $p$, which is chosen to be small enough that any valley or mountain which intersects $U_p$ passes through $p$. Then there exists a vector $\boldsymbol{\gamma} \in \mathbb R^d$ such that for any $q \in U_p$ not lying on a valley or mountain, we have
    $$\nabla h(q) + \nabla h(2p-q) = \boldsymbol{\gamma}.$$
\end{lemma}

\begin{proof}
    We translate so that $p$ is the origin, and call $U$ the small ball around the origin. For any point $q \in U$ not lying on a valley or mountain, let $c(q)$ be the $\{-1,1\}^{s+t}$ vector giving the sign of each $L_i$ evaluated at $q$. We make two simple observations; firstly, the entries of $c(q)$ corresponding to the $L_i$ whose zero sets do not intersect $U$, are constant throughout $U$. Secondly, for any linear function $L_i$ which does pass through $0$, $c(q)$ and $c(-q)$ are of opposite sign on the entry for $L_i$. We then conclude that the sum $c(q) + c(-q)$ does not depend on $q$ itself - it takes value $0$ on the entries corresponding to $L_i$ which pass $0$; all other entries were constant throughout $U$ and so these entries are simply doubled. This implies that the sum of local representations of $h$ at $q$ and $-q$ is constant in $U$, and therefore so are their gradients.
\end{proof}

We remark that in this proof, the size of the ball $U_p$ needed depends only on where the mountains and valleys not passing through $p$ lie; if \textit{all} mountains and valleys pass through $p$, then the ball $U_p$ can be skipped, and the lemma applies for any $q \in \mathbb R^d$ not lying on a valley or mountain.

\section{Examples}

Before proving our main results, we give some examples of hinge functions, and their corresponding positivity sets. Our first example is the interior of a triangle. Despite being the simplest polygon, it is not trivial to find a hinge function representing a triangle - in contrast to, for instance, a square, which can be easily given by $1-|x|-|y|$. In fact, the construction we use for a triangle requires five mountains/valleys.

Assume that the vertices of the triangle are $(0,0)$, $(2,0)$, and $(2,0)$. We need to have a mountain through each vertex. The linear functions defining the mountains can be chosen to be $L_1(x,y)=x-y$, $L_2(x,y)=2x+y-2$, and $L_3(x,y)=x+2y-2$. The mountain through one vertex intersects the opposite side, and there must be at least one valley through each intersection point. So we choose $L_4(x,y)=x-1$ and $L_5(x,y)=y-1$ for the valleys.

The hinge function can now be found by solving a linear system of equations expressing the condition that a general linear combination of $1$, $x$, $y$, $|L_1(x,y)|,\dots,|L_5(x,y)|$ vanishes at the the three vertices and the three intersection points. One solution is
\[ h(x,y) = x+y-2|x-y|-|x+2y-2|-|2x+y-2|+2|x-1|+2|y-1|. \]
Fortunately (or rather: by careful choice of the solution), the coefficients of mountains are negative and the coefficients of valleys are positive. This hinge function is shown below.
\begin{figure}[ht]
    \centering
    \includegraphics[width=0.45\linewidth]{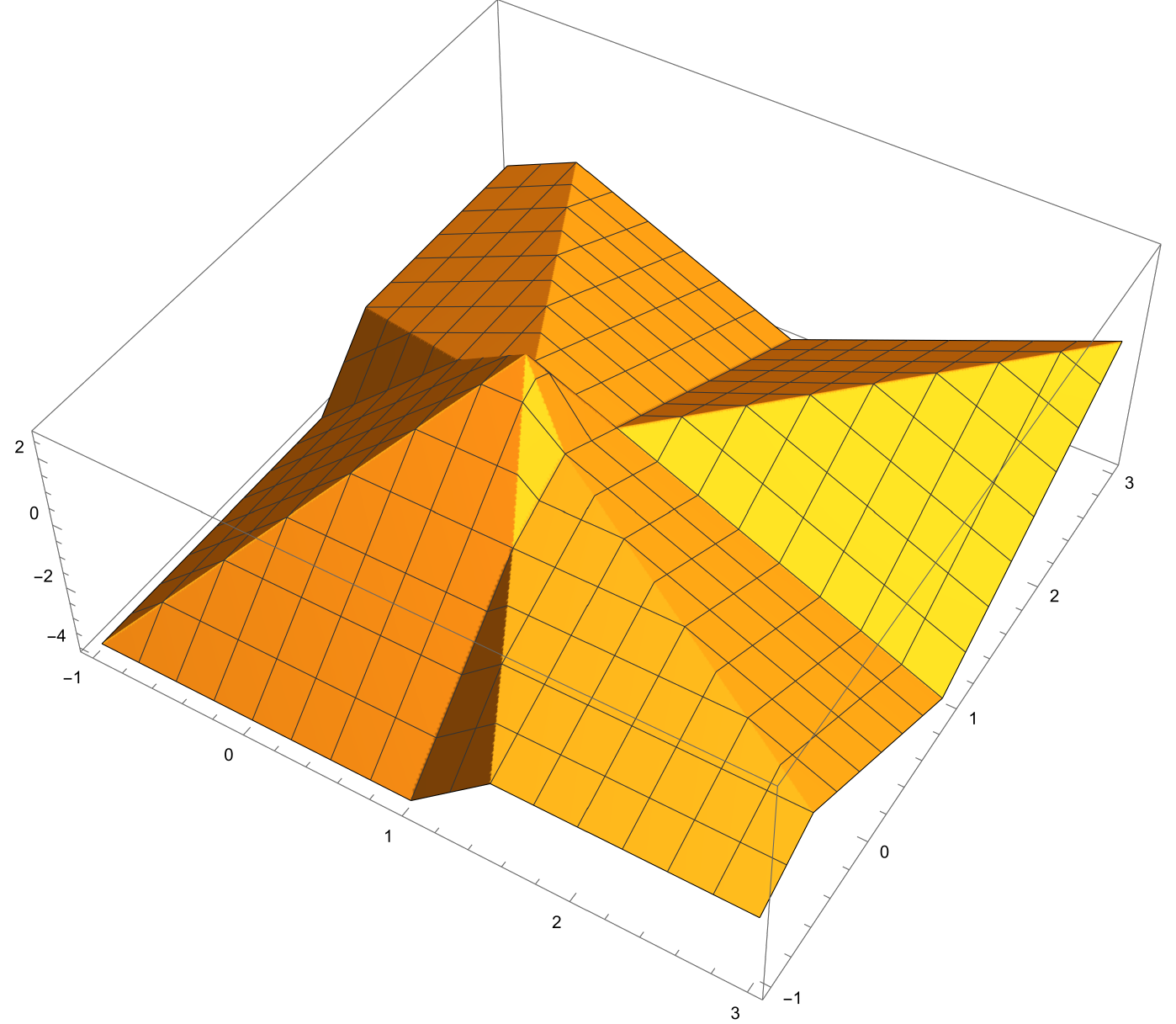}
    \hspace{15mm}
    \includegraphics[width=0.4\linewidth]{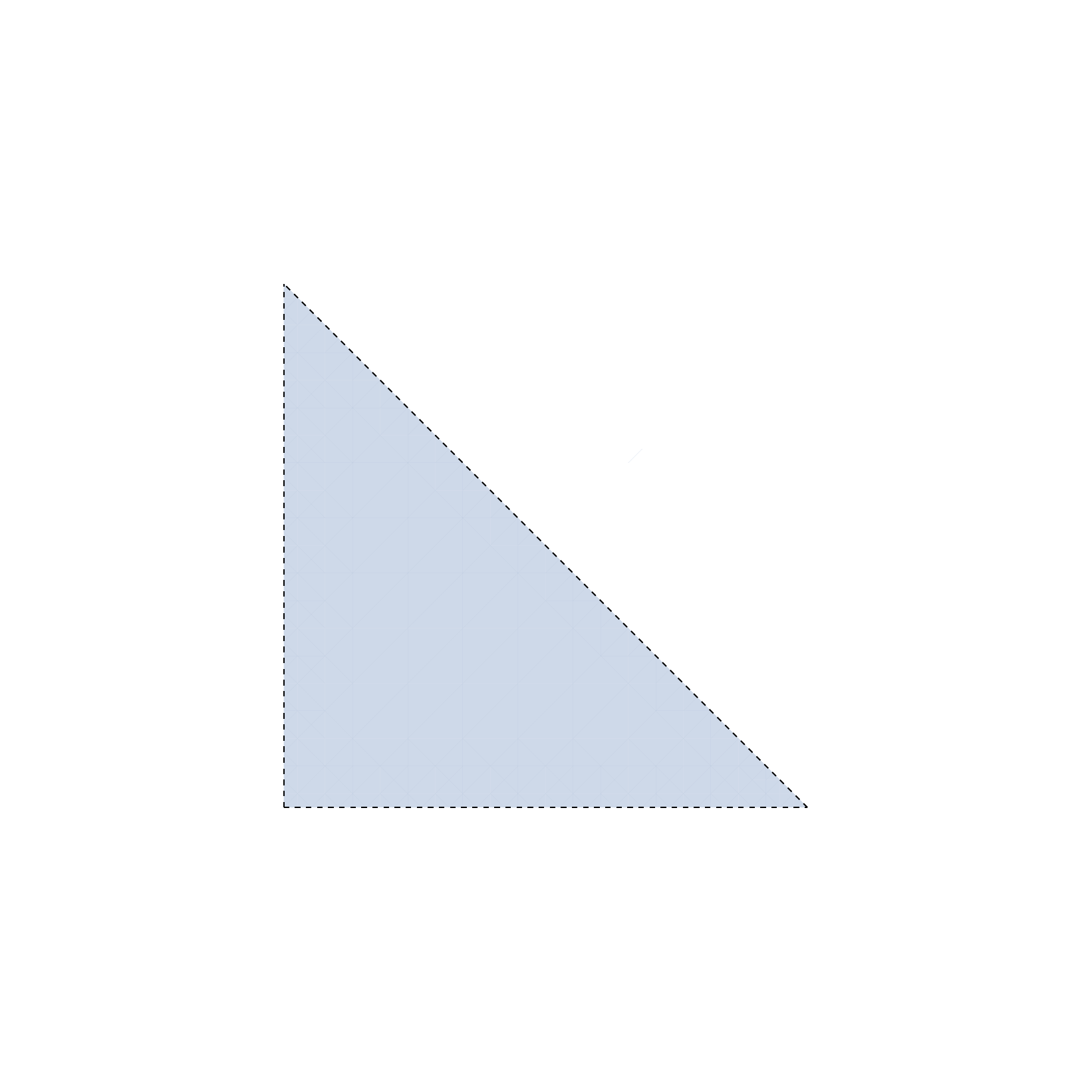}
    \caption{The hinge function $h$ whose positivity set is a triangle.}
    \label{fig:triangle}
\end{figure}

As a second example, we give a hinge function whose positivity set consists of four connected components. Such examples can be extended to any (even) number of connected components. The hinge function is given by the equation
\[h(x,y) = |x| + |y| + |y-2x| +|2y-x| - |y-3x| - |3y-x|-1.\]
This function, along with its positivity set, is shown in Figure \ref{fig:4-fan}.
\begin{figure}[ht]
    \centering
    \includegraphics[width=0.48\linewidth]{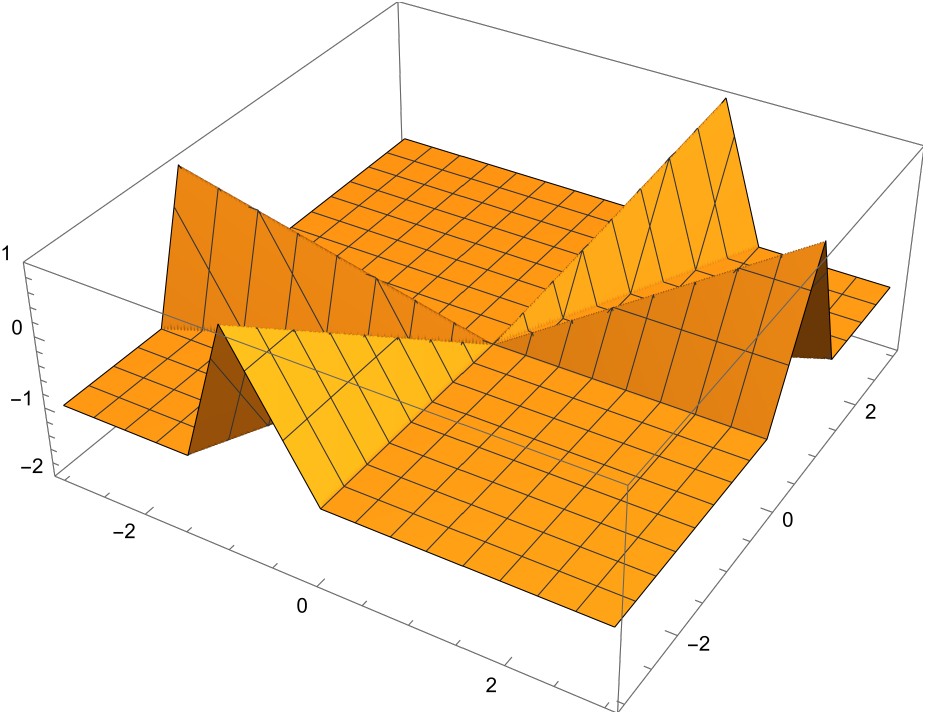}
    \hspace{15mm}
    \includegraphics[width=0.4\linewidth]{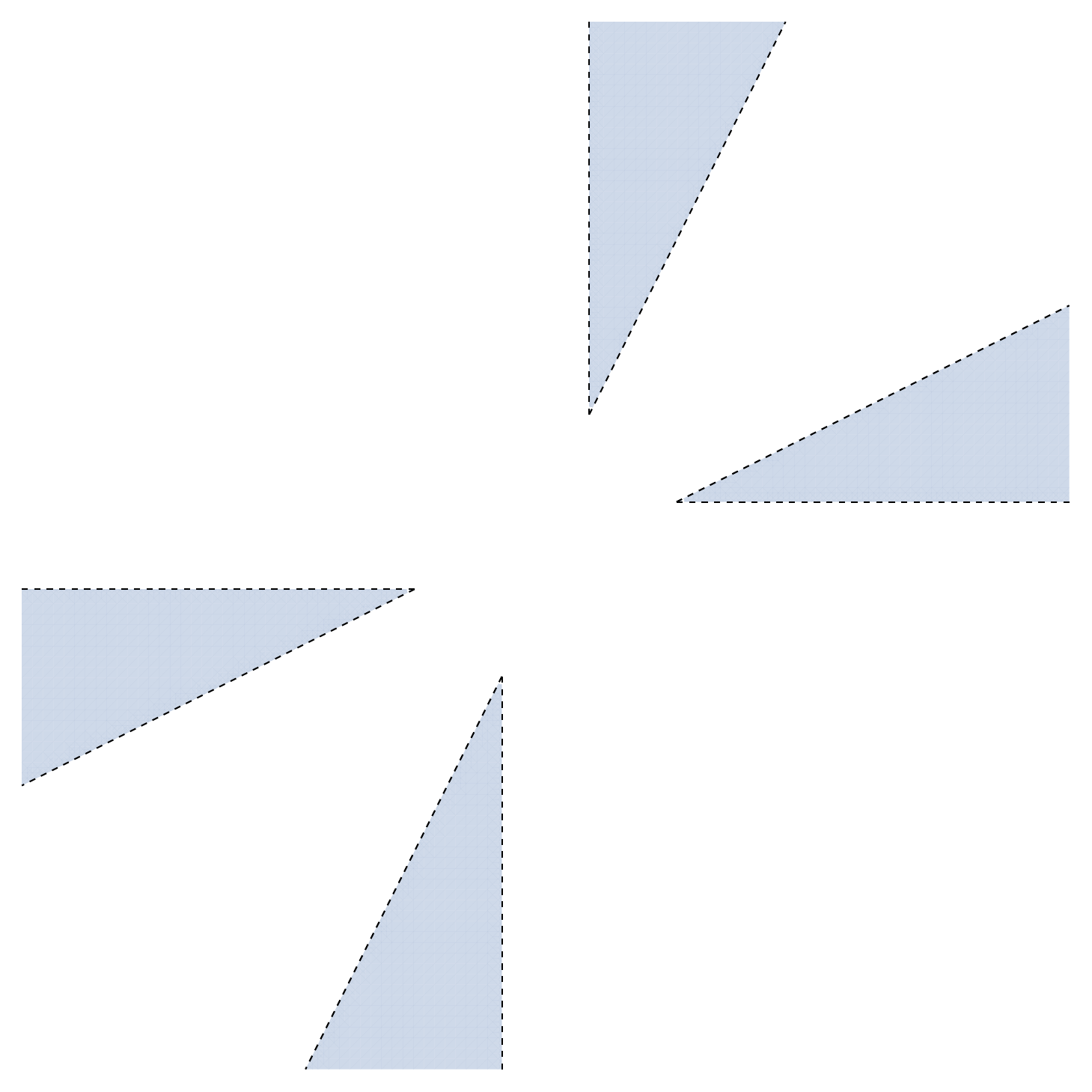}
    \caption{A hinge function whose positivity set has four connected components.}
    \label{fig:4-fan}
\end{figure}

\section{A local condition for positivity sets} \label{sec:localcond}

In this section we prove Theorem \ref{thm:localcond}. The case $d=2$ gives one direction of Theorem \ref{thm:main}. 

\begin{proof}[Proof of Theorem \ref{thm:localcond}]
    Suppose that $P$ is a polytopal subset of $\mathbb R^d$, and that $q \in \mathbb R^d$ is such that $C_q$ gives a non-empty intersection $\langle G \rangle \cap \conv (R)$. Let $\alpha$ be a point in this intersection. Since the convex hull of $R$ is the union of all simplices with vertices in $R$, there exists some simplex with vertices $\{v_1,...,v_{d+1}\} \subseteq R$, such that $\alpha$ lies in the convex hull of $V$. Therefore, there exist real numbers $\lambda_1,...,\lambda_{d+1} \in \mathbb R_{\geq0}$, which are not all zero, such that $\sum_{i=1}^{d+1} \lambda_i v_i = \alpha$. Furthermore, since $\alpha \in \langle G \rangle$, there exist $\mu_1,...,\mu_n \in \mathbb R$ and $g_1,...,g_n \in G$ such that $\sum_{i=1}^n\mu_i g_i =\alpha$.

    For contradiction, assume that $P$ is the positivity set of a hinge function $h_0$. Such a hinge function yields a hinge function for the local cone $C_q$. To see this, we first assume, without loss of generality, that $q$ is the origin. Assume $h=L_0+\sum_{i=1}^s |L_i|-\sum_{j=s+1}^t |L_j|$. Any of the $L_i$ or $L_j$ that does not vanish at $q$ assumes a constant nonzero sign in a neighborhood of $q$. Since we only care about small neighborhoods, we take these summands out of the sum and add it to $L_0$, with the correct sign. If $L_0(q)\ne 0$ for the updated $L_0$, then the cone is either empty or equal to $\R^d$ -- these cases are trivial and we do not consider them any further. Otherwise, this gives a hinge function $h$ which has positivity set $C_q$. 
    
    Consider any $v_i \in R$. If $v_i$ does not lie on a mountain or valley of $h$, then we have $ \nabla h(v_i)\cdot v_i=h(v_i)   >0$, since $v_i \in C$. Furthermore, we must also have $ \nabla h(-v_i) \cdot (-v_i)=h(-v_i)  \leq 0$. By adding the first equation to the negation of the second, we then have
    $$(\nabla h (v_i) + \nabla h (-v_i)) \cdot v_i > 0.$$
    By Lemma \ref{lem:localgradient}, there exists some vector $\boldsymbol{\gamma}$ such that for all $v_i$ not on a valley or mountain, we have $\nabla h (v_i) + \nabla h (-v_i) = \boldsymbol{\gamma}$, and therefore
    $$\boldsymbol{\gamma} \cdot v_i > 0.$$
    We now wish to derive the same equation for $v_i$ which \textit{do} lie on a valley or mountain. To do this, we first prove the equation $h(p)-h(-p)=2\boldsymbol{\gamma}\cdot p$. If $p$ is not contained in mountain and not in a valley, then $h(p)=\nabla h(p)\cdot p$ and $h(-p)=\nabla h(-p)\cdot(-p)$, and the equation follows. If $p$ lies on a mountain or a valley, then the equation follows by continuity. Returning to $v_i$ and using $h(v_i)>0$ and $h(-v_i)\le 0$, we get $\boldsymbol{\gamma}\cdot v_i>0$.
    
    We can now finish the proof. We consider the dot product of $\boldsymbol{\gamma}$ and the linear combination $\sum_{i=1}^{d+1} \lambda_i v_i = \alpha$. We have
    $$0 < \sum_{i=1}^{d+1} \lambda_i (\boldsymbol{\gamma} \cdot v_i) =\boldsymbol{\gamma} \cdot \sum_{i=1}^{d+1} \lambda_i v_i = \boldsymbol{\gamma} \cdot \alpha =\sum_{i=1}^n\mu_i (\boldsymbol{\gamma} \cdot g_i).$$
    where the first inequality holds since $\boldsymbol{\gamma} \cdot v_i >0$, and not all of the $\lambda_i$ are zero. Pick any dot product $\boldsymbol{\gamma} \cdot g_i$. As $g_i$ and $-g_i$ both lie on the boundary of the positivity set, from the equation above we have that
    $$\boldsymbol{\gamma} \cdot g_i = \frac{h(g_i)+h(-g_i)}{2}=0 . $$
    Therefore the entire sum above is equal to zero, giving a contradiction.
\end{proof}

 An example of a cone which \textit{fails} to satisfy the local condition, and is therefore not a positivity set, is shown in Figure \ref{fig:3fan}; we call this set the \textit{3-fan}. The three dotted lines shown in the picture are given by the equations $y=0$, $y=x$, and $y=-x$. In this example, the set $R$ as defined by Proposition \ref{thm:localcond} is the cone itself, and since $G$ contains three entire lines we have $\langle G \rangle = \mathbb R^2$.

 \begin{figure}[ht]
     \centering
     \includegraphics[width=0.4\linewidth]{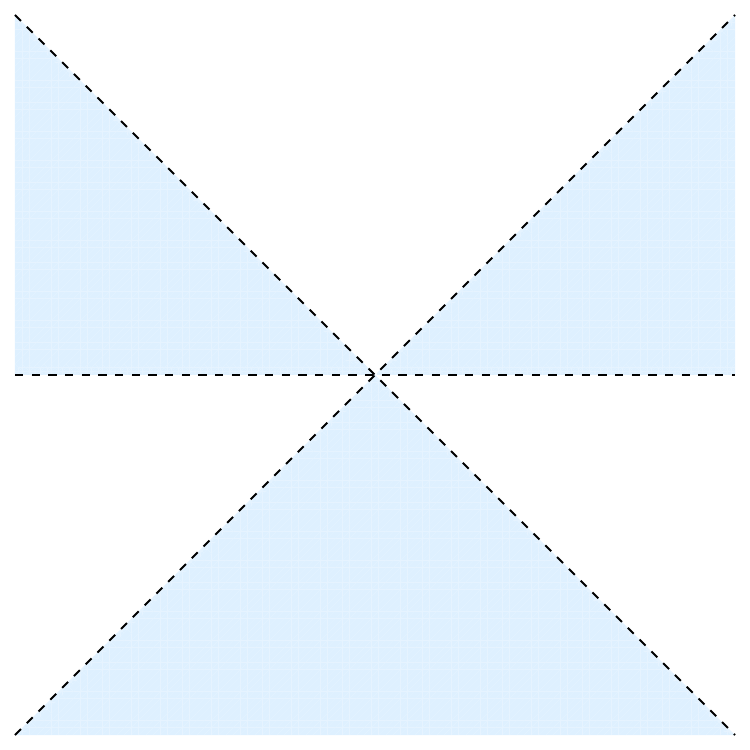}
     \caption{The 3-fan is not a positivity set.}
     \label{fig:3fan}
 \end{figure}

\begin{example}\label{eg:n_fan}
We consider a polytopal cone $C \subseteq \mathbb{R}^{2}$ which is a \textit{k-fan}, where $k = 2m, m \in \mathbb{N}$, i.e., $C$ is a union of an even number of wedges sharing the origin in common. As $C$ is a centrally symmetric set, therefore by Theorem \ref{thm:main} there exists a hinge function $f$ whose positivity set is $C$. \end{example}

In \cite{grigsby2022transversality} the authors show that for 1-layer ReLU networks corresponding to hinge functions $h:  \mathbb{R}^{n} \rightarrow \mathbb{R}$ that do not allow skip connections and whose hidden layer has dimension $n+1$, their positivity set can have at most one bounded connected component. Example \ref{eg:n_fan} shows that these assumptions are necessary for this claim, as in our general setting where skip connections are allowed, with a suitable choice of $k$ additional valleys, from the \textit{k-fan} we can construct a positivity set that has $k, k >1$ bounded connected components. 

\section{Other direction of Theorem \ref{thm:main}}

We will now prove the remaining direction of Theorem \ref{thm:main}, in the form of the following proposition.

\begin{prop}
    Let $C \subseteq \mathbb R^2$ be a polytopal cone. If we have $\langle G \rangle \cap \conv (R) = \emptyset$, then $C$ is the positivity set of some hinge function.
\end{prop}

In order to prove this proposition, we will require the following lemma. We remark that by `symmetric function' we mean a function $f( \boldsymbol{x})$ such that $f( \boldsymbol{x}) = f( -\boldsymbol{x})$ for all $\boldsymbol{x}$, that is, we mean a \textit{centrally symmetric} function.

\begin{lemma}\label{lem:classification}
    Let $f: \mathbb R^2 \rightarrow \mathbb R$ be a positively homogeneous, continuous, piecewise linear function. Then $h$ is a hinge function if and only if there exists a continuous, piecewise linear, symmetric, positively homogeneous function $s$ and a vector $\boldsymbol{e}$ such that $f(\boldsymbol{x}) = s(\boldsymbol{x}) + \boldsymbol{e} \cdot \boldsymbol{x}$.
\end{lemma}

\begin{proof}
    We begin by showing that any such hinge function can be written as $s(\boldsymbol{x}) + \boldsymbol{e}\cdot \boldsymbol{x}$. Indeed, any positively homogeneous hinge function must have all of its mountains and valley pass through the origin; if some mountain/valley did not pass the origin, the change in gradient would contradict positive homogeneity of $f$. Therefore when we write $f$ in the form $f(\boldsymbol{x}) = L_0(\boldsymbol{x}) + \sum_{i=1}^s |L_i(\boldsymbol{x})| - \sum_{j=s+1}^t |L_j(\boldsymbol{x})|$, each $L_i$ is a linear form. But we are now done, as the $L_0(\boldsymbol{x})$ can be written as $\boldsymbol{e} \cdot \boldsymbol{x}$ for some $\boldsymbol{e}$, and the sum of absolute values is a continuous, piecewise linear, symmetric, and positively homogeneous function.

    We now prove the other direction. We take $f$ which can be written as $s(\boldsymbol{x}) + \boldsymbol{e} \cdot \boldsymbol{x}$. Since $\boldsymbol{e} \cdot \boldsymbol{x}$ is itself a hinge function, we need only show that $s(\boldsymbol{x})$ is a hinge function; their sum is then the hinge function needed. Since $s(\boldsymbol{x})$ is symmetric and positively homogeneous, we have that $s(\boldsymbol{0})=0$, and its points of non-differentiability form lines through the origin. These lines will be the valleys/mountains of the hinge function. We now apply induction on the number of such lines of non-differentiability. If no such lines exist, then $s(\boldsymbol{x})$ is just a linear function, which is a hinge function - this is our base case. Now assume that any such function with at most $k-1$ lines of non-differentiability is in fact a hinge function. Suppose that $s$ has $k$ lines of non-differentiability. Pick any one of them; by rotating our coordinate axis, suppose it is the line $y=0$. The change in gradient of the function $s$ when passing this line is of the form $(0,\lambda)$ for some $\lambda \neq 0$. If we consider the function $s(\boldsymbol{x}) - \frac{\lambda}{2}|y|$, we notice that it has one less line of non-differentiability, namely the $y=0$ line is now differentiable. Indeed, suppose WLOG that $s$ has gradient $(a,b)$ just below the $y$-axis, and $(a,b+\lambda)$ just above. Then $s(\boldsymbol{x}) - \frac{\lambda}{2}|y|$ has gradient $(a,b + \lambda/2)$ just above and just below the $y$-axis. By induction, $s(\boldsymbol{x}) - \frac{\lambda}{2}|y|$ is indeed a hinge function, and we recover $s(\boldsymbol{x})$ as a hinge function by adding $ \frac{\lambda}{2}|y|$, which is itself a hinge function. This concludes the proof.
\end{proof}
    We remark that the classification given by \Cref{lem:classification} fails to hold in higher dimensions. For instance, consider the symmetric function on $\mathbb R^3$ given by $f(x,y,z) = \max\{x,y,z,-x,-y,-z,0\}$. If $f$ could be written as a hinge function, then the function $g(x,y):= f(x,y,1+x+y)$ would also be a hinge function. On the other hand, the function $g$ has points of non-differentiability which do not form full lines, and therefore cannot be given by a hinge function. The graph of $g$ is shown in \Cref{fig:nonexamplehinge}.

    \begin{figure}[h]
        \centering
        \includegraphics[width=0.47\linewidth]{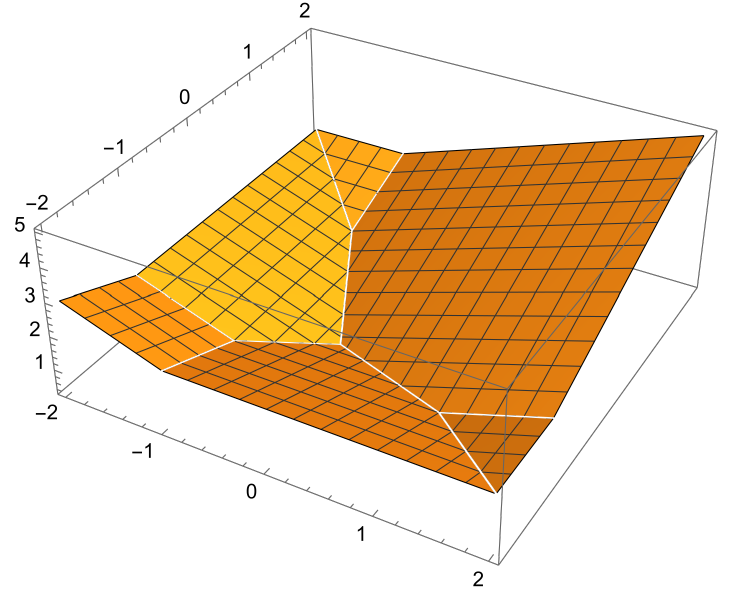}
        \caption{The non-differentiable points of $g$ do not form full lines.}
        \label{fig:nonexamplehinge}
    \end{figure}

\subsection{Constructing values of a symmetric function along rays}

Given Lemma \ref{lem:classification}, we aim to show the existence of a continuous, piecewise linear, positively homogeneous symmetric function $s$ and vector $\boldsymbol{e}$ such that $C$ is the positivity set of $s(\boldsymbol{x}) + \boldsymbol{e} \cdot \boldsymbol{x}$. We begin with the choice of $\boldsymbol{e}$.

\subsubsection{\texorpdfstring{Choice of $\boldsymbol{e}$ and open half plane $\pi$}{}}

If $\langle G \rangle$ is either just the origin itself or a single line, then the assumption that $\langle G \rangle \cap \conv(R)=\emptyset$ implies that $R$ is contained in some open half plane whose boundary passes the origin. Pick any such half plane, and call it $\pi$. We will then choose $\boldsymbol{e}$ to be the unit vector normal to the boundary line of $\pi$, whose direction points into $\pi$. If we have $\langle G \rangle = \mathbb R^2$, then $R$ must be empty, and in this case we will choose $\boldsymbol{e} = \boldsymbol{0}$, and $\pi$ can be chosen as any half plane bounded by a line contained in $G$. These three cases are shown in the table below.
\begin{table}[ht]
    \centering
    \begin{tabular}{|c|c|c|}
    \hline
        $\langle G \rangle = \{\boldsymbol{0}\}$ & $\langle G \rangle$ is a line $l$ &  $\langle G \rangle = \mathbb R^2$, $G$ contains line $l$\\

        $R$ in half-plane $\pi$ & $R$ in half-plane $\pi$ bounded by $l$ & $R = \emptyset$, $\pi$ bounded by $l$\\ 

        $\boldsymbol{e}$ inward normal to $\pi$ & $\boldsymbol{e}$ inward normal to $\pi$ & $\boldsymbol{e} = 0$ \\ \hline
    \end{tabular}
    \caption*{The cases for $R$, $\pi$, and $\boldsymbol{e}$ depending on the dimension of $\langle G \rangle$.}
\end{table}

\subsubsection{\texorpdfstring{Constructing the symmetric function $s(\boldsymbol{x})$}{}}

We must now construct a continuous piecewise linear symmetric function $s( \boldsymbol{x})$, and show that the resulting function $f(\boldsymbol{x}):=s( \boldsymbol{x}) + \boldsymbol{e} \cdot \boldsymbol{x}$ has $C$ as its positivity set. Notice that if $\boldsymbol{x}$ lies on any boundary ray of $C$, we wish to have function value $0$ and so
$$0 = s( \boldsymbol{x}) + \boldsymbol{e} \cdot \boldsymbol{x} \implies s( \boldsymbol{x}) = -\boldsymbol{e} \cdot \boldsymbol{x}$$
therefore the values of our symmetric function $s$ are known on any boundary ray of $C$. Since $s$ is symmetric, the value of $s$ is actually known not only on the rays, but the full lines given by the rays. These lines split the half plane $\pi$ into segments $S_1,\dots,S_t$, where $S_1$ and $S_t$ share are the extremal segments (see Figure \ref{fig:segmentsofpi} for an illustration). For any segment $S_i$, there are three possibilities for points in the interior of $C$, which we name cases $1$, $2$, and $3$.
\begin{description}
    \item[Case 1.] We may have that for all $\boldsymbol{p} \in S_i$, both $\boldsymbol{p}$ and $-\boldsymbol{p}$ lie in the cone $C$. In this case we need to have both $f(\boldsymbol{p})$ and $f(-\boldsymbol{p})$ greater than zero. This gives us the condition that within such a segment we must have $s(\boldsymbol{p}) > |\boldsymbol{e} \cdot \boldsymbol{p}| = \boldsymbol{e} \cdot \boldsymbol{p}$ (recall that within the half plane $\pi$ we always have $\boldsymbol{e} \cdot \boldsymbol{p} \geq 0$, with strict inequality if $\boldsymbol{e} \neq \boldsymbol{0}$).

    \item[Case 2.] Secondly, we may have the opposite case that for all $\boldsymbol{p} \in S_i$, both $\boldsymbol{p}$ and $-\boldsymbol{p}$ are not in $C$. We must then have that both $f(\boldsymbol{p})$ and $f(-\boldsymbol{p})$ are negative or zero, leading to the condition $s(\boldsymbol{p}) \le |\boldsymbol{e} \cdot \boldsymbol{p}| = -\boldsymbol{e} \cdot \boldsymbol{p}$.

    \item[Case 3.] Lastly, we may have that for all $\boldsymbol{p} \in S_i$, $\boldsymbol{p}$ lies in $C$, but $-\boldsymbol{p}$ is \textit{not} in $C$. Note that such points are in the set $R$, so this case does not occur if $\boldsymbol{e} = \boldsymbol{0}$. At these points we must have $f(\boldsymbol{p}) > 0$ and $f(-\boldsymbol{p}) \leq 0$, which leads to the condition that $-\boldsymbol{e} \cdot \boldsymbol{p} < s(\boldsymbol{p}) \leq \boldsymbol{e} \cdot \boldsymbol{p}$.
\end{description}

\begin{table}[ht]
    \centering
    \begin{tabular}{|c|c|c|}
    \hline
        Condition in Case 1 & Condition in Case 2  &  Condition in Case 3\\ \hline
        $s(\boldsymbol{p}) > \boldsymbol{e} \cdot \boldsymbol{p}$ & $s(\boldsymbol{p}) \le -\boldsymbol{e} \cdot \boldsymbol{p}$ & $-\boldsymbol{e} \cdot \boldsymbol{p} < s(\boldsymbol{p}) \leq \boldsymbol{e} \cdot \boldsymbol{p}$\\ \hline
    \end{tabular}
    \caption*{The conditions which must be satisfied by $s$ in each case.}
\end{table}

\begin{figure}[!ht]
    \centering
    \includegraphics[width=0.45\linewidth]{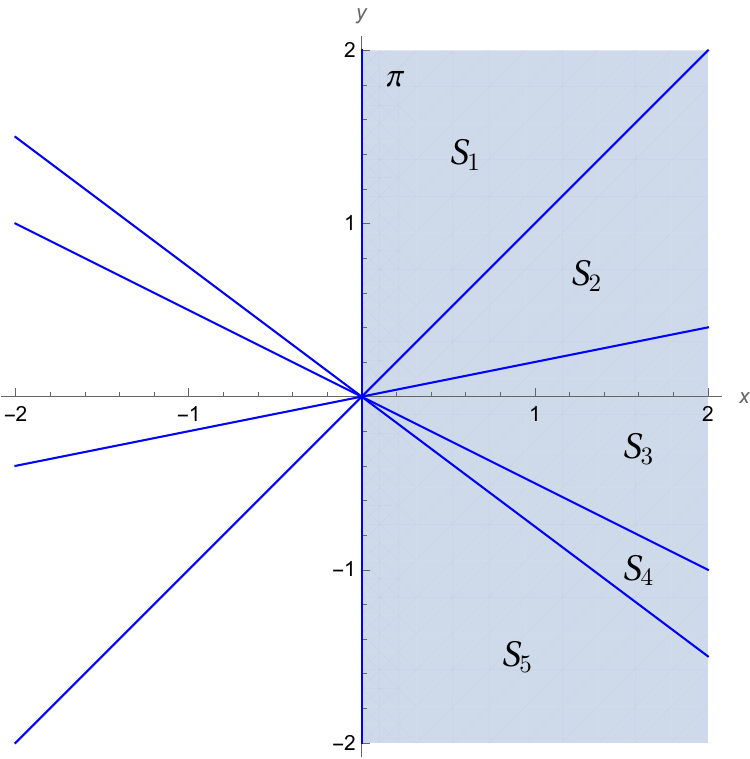}
    \caption{The lines given by boundary rays of $C$ split the half plane $\pi$ into segments.}
    \label{fig:segmentsofpi}
\end{figure}

Note that there is no case where $\boldsymbol{p} \notin C$ and $-\boldsymbol{p} \in C$, as this would mean that $-\boldsymbol{p} \in R$, contradicting the fact that $R$ is contained in the half plane $\pi$. We now consider the boundary rays of these segments which are \textit{not} the boundary rays of $\pi$ (these will be dealt with separately). These `interior' boundary rays can also be split into three cases, which we name \textbf{A}, \textbf{B}, and \textbf{C}.

\begin{description}
    \item[Case A.] We may have a boundary ray $B$ of $S_i$ such that for all $\boldsymbol{p} \in B$, we have $\boldsymbol{p} \in \partial C$, and $-\boldsymbol{p} \notin \partial C$. In this case we have $s(\boldsymbol{p}) = - \boldsymbol{e} \cdot \boldsymbol{p}$. We remark that if this case occurs, the set $R$ is non-empty (and therefore $\boldsymbol{e} \neq \boldsymbol{0}$). This is seen by taking some small neighbourhood of $\boldsymbol{p}$, which must intersect $C$ and gives points from $R$.
    \item[Case B.] We may have a boundary ray $B$ of $S_i$ such that for all $\boldsymbol{p} \in B$, we have $\boldsymbol{p} \in C$, and $-\boldsymbol{p} \in \partial C$. On such boundary rays we have $s(\boldsymbol{p}) =  \boldsymbol{e} \cdot \boldsymbol{p}$.
    \item[Case C.] We may have interior boundary rays such that both $\boldsymbol{p}$ and $-\boldsymbol{p}$ lie on $\partial C$. If $R$ is non-empty, then this case cannot occur on these interior boundary rays. If this case does occur, then $\boldsymbol{e}=\boldsymbol{0}$, and all points $\boldsymbol{p}$ on these rays must have $s(\boldsymbol{p}) = 0$.
\end{description}
\begin{table}[ht]
    \centering
    \begin{tabular}{|c|c|c|}
    \hline
         Definition in Case \textbf{A} & Definition in Case \textbf{B}  &  Definition in Case \textbf{C}\\ \hline
        $s(\boldsymbol{p}) = - \boldsymbol{e} \cdot \boldsymbol{p}$ & $s(\boldsymbol{p}) =  \boldsymbol{e} \cdot \boldsymbol{p}$ & $s(\boldsymbol{p}) = 0$\\ \hline
    \end{tabular}
    \caption*{Definition of $s(\boldsymbol{p})$ for $\boldsymbol{p}$ in an interior boundary ray.}
\end{table}
We remark that similarly to above, the case where $\boldsymbol{p} \notin C$ and $-\boldsymbol{p}\in \partial C$ cannot occur, as taking a point close enough to $-\boldsymbol{p}$ which lies within $C$ would yield a point of $R$ not lying in the half plane $\pi$. 

We now split each segment $S_i$ into two, by placing an extra ray, call it $r_i$, inside each one (see Figure \ref{fig:segmentswithrays}). We will define the value of $s$ along each such ray. Depending on whether $S_i$ lies in case 1, 2, or 3, we define the value of $s$ along $r_i$ differently. 
\begin{enumerate}
    \item If $S_i$ is in case $1$, we pick any vector $\boldsymbol{v}$ which is in a common open half-plane with $r_i$ (so that $\boldsymbol{v} \cdot \boldsymbol{p} >0$ for $\boldsymbol{p} \in r_i$). We then define $s(\boldsymbol{p}) = (\boldsymbol{v}+\boldsymbol{e}) \cdot \boldsymbol{p}$ for $\boldsymbol{p} \in r_i$. This choice satisfies condition $1$ along the ray $r_i$. Note that if $\boldsymbol{e} \neq \boldsymbol{0}$ then $\boldsymbol{v}=\boldsymbol{e}$ will always work.
    \item If $S_i$ is in case $2$, we similarly pick any vector $\boldsymbol{v}$ which is in a common open half-plane with $-r_i$ (so that $\boldsymbol{v} \cdot \boldsymbol{p} <0$ for $\boldsymbol{p} \in r_i$). We then define $s(\boldsymbol{p}) = (\boldsymbol{v} - \boldsymbol{e}) \cdot \boldsymbol{p}$ for $\boldsymbol{p} \in r_i$. This choice satisfies condition $2$ along the ray $r_i$. Note that if $\boldsymbol{e} \neq \boldsymbol{0}$ then $\boldsymbol{v}=-\boldsymbol{e}$ will always work.
    \item If $S_i$ is in case $3$, we define that for $\boldsymbol{p} \in r_i$, we have $s(\boldsymbol{p}) = 0$. This satisfies condition $3$ along $r_i$ (recall that $\boldsymbol{e} \neq \boldsymbol{0}$ if this case occurs.)
\end{enumerate}
\begin{table}[ht]
    \centering
    \begin{tabular}{|c|c|c|}
    \hline
         Definition in Case 1 & Definition in Case 2  &  Definition in Case 3\\ \hline
        $s(\boldsymbol{p}) = (\boldsymbol{v}+\boldsymbol{e}) \cdot \boldsymbol{p}$ & $s(\boldsymbol{p}) = (\boldsymbol{v} - \boldsymbol{e}) \cdot \boldsymbol{p}$  & $s(\boldsymbol{p}) = 0$\\ 
        $\boldsymbol{v}\cdot \boldsymbol{p} >0$ & $\boldsymbol{v}\cdot \boldsymbol{p} <0$ & $\boldsymbol{e} \neq \boldsymbol{0}$\\
        \hline
    \end{tabular}
    \caption*{Definition of $s(\boldsymbol{p})$ for $\boldsymbol{p}$ in a ray $r_i \subseteq S_i$.}
\end{table}
\begin{figure}[ht] 
    \centering
    \includegraphics[width=0.5\linewidth]{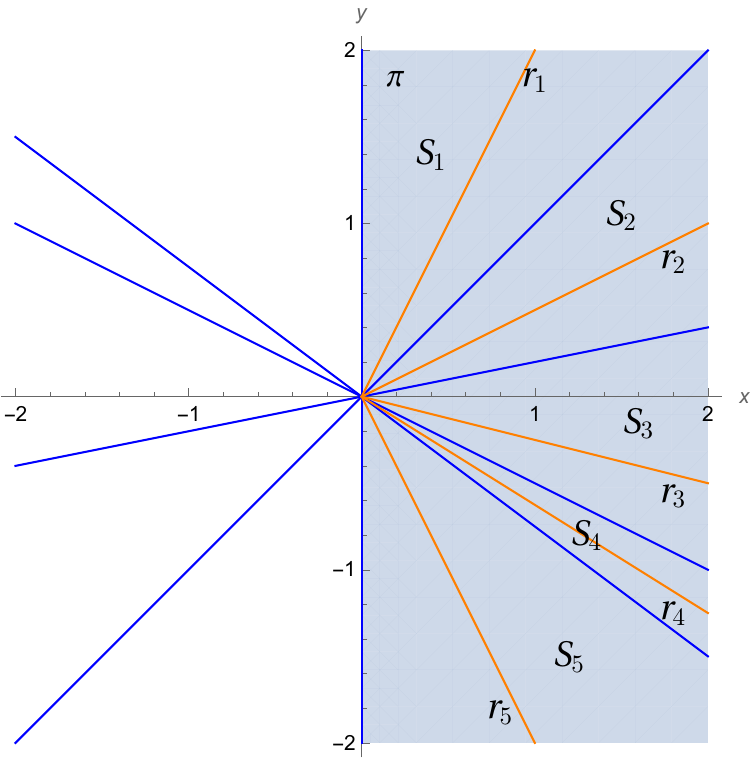}
    \caption{The slices $S_i$ shown with the rays $r_i$}
    \label{fig:segmentswithrays}
\end{figure}

We now interpolate linearly between each $r_i$ with its assigned linear function, and the two boundary rays of $S_i$ with their assigned linear functions coming from cases $\textbf{A}, \textbf{B}$, and $\textbf{C}$. 

\subsection{Case analysis for the interpolation}

Take any $r_i$, and one of the boundary rays $B$ of $S_i$. for simplicity we drop subscripts, setting $S:=S_i$ and $r := r_i$. Depending on the case which $S$ lies in, we need to show that with our definitions and linear interpolation, the corresponding condition needed for that case holds within $S$. We have various combinations of cases to analyse, for $S$ being in case $1,2,3$, and the boundary ray $B$ being in case \textbf{A}, \textbf{B}, or \textbf{C}. Throughout these cases we will write $\boldsymbol{p} = \lambda \boldsymbol{p}_1 + \mu \boldsymbol{p}_2$ with $\lambda, \mu>0$ and $\boldsymbol{p}_1 \in B$ and $\boldsymbol{p}_2 \in S$.

\begin{description}
\item[Case 1\textbf{A} does not occur.] 
    If $S$ is in case $1$, the boundary ray $B$ cannot be in case \textbf{A}. Indeed, if $B$ were in case \textbf{A}, then the ray $-B$ is not in $\partial C \cup C$, but is a boundary ray of $-S_i \subseteq C$, giving a contradiction. 
\item[Case 1\textbf{B}.]Suppose $S$ is in 
    case $1$, and $B$ is in case \textbf{B}. We then have $s(\boldsymbol{p}_1) = \boldsymbol{e} \cdot \boldsymbol{p}_1$ for $\boldsymbol{p}_1 \in B$, and $s(\boldsymbol{p}_2) = (\boldsymbol{v} +\boldsymbol{e} ) \cdot \boldsymbol{p}_2$ for $\boldsymbol{p}_2 \in r$. We then have
    \begin{align*}
        s(\lambda \boldsymbol{p}_1 + \mu \boldsymbol{p}_2) = \lambda (\boldsymbol{e} \cdot \boldsymbol{p}_1) + \mu (\boldsymbol{v} +\boldsymbol{e} ) \cdot \boldsymbol{p}_2 > \lambda (\boldsymbol{e} \cdot \boldsymbol{p}_1) + \mu (\boldsymbol{e} \cdot \boldsymbol{p}_2) = \boldsymbol{e} \cdot \boldsymbol{p}
    \end{align*}
    and so condition $1$ is satisfied by our linear interpolation within $S$.
\item[Case 1\textbf{C}.] Now suppose that 
    $B$ is in case \textbf{C}. Any such occurrence of case \textbf{C} not on the boundary rays of $\pi$ means that $\boldsymbol{e}=\boldsymbol{0}$, and so $s(\boldsymbol{p}_1) = 0$. We then have
    $$s(\lambda \boldsymbol{p}_1 + \mu \boldsymbol{p}_2) = \mu(\boldsymbol{v} +\boldsymbol{e} ) \cdot \boldsymbol{p}_2 > \mu(\boldsymbol{e}  \cdot \boldsymbol{p}_2) =0= \boldsymbol{e} \cdot \boldsymbol{p}.$$
    Therefore condition $1$ is satisfied for all points in $S$.
\item[Case 2\textbf{A}.] If $B$ is in 
    case \textbf{A}, then we have $s(\boldsymbol{p}_1)= -\boldsymbol{e}\cdot \boldsymbol{p}_1$ for $\boldsymbol{p}_1 \in B$. If $S$ is in case $2$, we have $s(\boldsymbol{p}_2) = (\boldsymbol{v} -\boldsymbol{e} ) \cdot \boldsymbol{p}_2$ for $\boldsymbol{p}_2 \in r$. We then have
    $$s(\lambda \boldsymbol{p}_1 + \mu \boldsymbol{p}_2) = \lambda (-\boldsymbol{e} \cdot \boldsymbol{p}_1) + \mu(\boldsymbol{v} -\boldsymbol{e} ) \cdot \boldsymbol{p}_2 < \lambda (-\boldsymbol{e} \cdot \boldsymbol{p}_1) + \mu(-\boldsymbol{e} ) \cdot \boldsymbol{p}_2 = -(\boldsymbol{e} \cdot \boldsymbol{p})$$
    showing that condition $2$ is satisfied within $S$. 
\item[Case 2\textbf{B} does not occur.] 
    This case cannot occur. Indeed, suppose it does, so that all points in $S$ are not in $C$. But since $B$ is inside $C$, there must be boundary of $C$ between $B$ and $S$, contradicting that $B$ is the boundary of $S$.
\item[Case 2\textbf{C}.] As before, any  
    occurrence of case \textbf{C} outside of the boundary rays of $\pi$ implies that $\boldsymbol{e} = \boldsymbol{0}$. If $B$ is in case \textbf{C}, we have that $s(\boldsymbol{p}_1)=0$ for all $\boldsymbol{p}_1 \in B$. Since $S$ is in case $2$, we have that $s(\boldsymbol{p}_2) = \boldsymbol{v} \cdot \boldsymbol{p}_2 < 0$ for $\boldsymbol{p}_2 \in r$. We then have
    $$s(\lambda \boldsymbol{p}_1 + \mu \boldsymbol{p}_2) = \mu (\boldsymbol{v} \cdot \boldsymbol{p})<0 = -\boldsymbol{e} \cdot \boldsymbol{p}$$
    so that condition $2$ is satisfied within $S$.
\item[Case 3\textbf{A}.] In this case, 
    we have that $s(\boldsymbol{p}_1) = -\boldsymbol{e}\cdot \boldsymbol{p}_1$ for $\boldsymbol{p}_1 \in B$. Since $S$ is in case $3$, we have that $s(\boldsymbol{p}_2) =0$. We then have
    $$s(\lambda \boldsymbol{p}_1 + \mu \boldsymbol{p}_2) = -\lambda (\boldsymbol{e}\cdot \boldsymbol{p}_1)>-\lambda (\boldsymbol{e}\cdot \boldsymbol{p}_1)  -\mu (\boldsymbol{e}\cdot \boldsymbol{p}_2) = - \boldsymbol{e}\cdot \boldsymbol{p}.$$
    Therefore the lower bound $s(\boldsymbol{p}) > -\boldsymbol{e}\cdot \boldsymbol{p}$ holds. Furthermore, we have
    $$s(\lambda \boldsymbol{p}_1 + \mu \boldsymbol{p}_2) = -\lambda (\boldsymbol{e}\cdot \boldsymbol{p}_1)<0<\boldsymbol{e}\cdot \boldsymbol{p}.$$
    Therefore the upper bound $s(\boldsymbol{p}) \leq \boldsymbol{e}\cdot \boldsymbol{p}$ holds, which together with the lower bound shows that condition $3$ holds within $S$.
\item[Case 3\textbf{B}.] In this case we 
    have that $s(\boldsymbol{p}_1) = \boldsymbol{e}\cdot \boldsymbol{p}_1$ for $\boldsymbol{p}_1 \in B$. We still have $s(\boldsymbol{p}_2) = 0$ for $\boldsymbol{p}_2 \in r$. We then have
    $$s(\lambda \boldsymbol{p}_1 + \mu \boldsymbol{p}_2) = \lambda (\boldsymbol{e}\cdot \boldsymbol{p}_1) > 0 > -\boldsymbol{e}\cdot \boldsymbol{p}$$
    giving the lower bound in condition $3$. We also have
    $$s(\lambda \boldsymbol{p}_1 + \mu \boldsymbol{p}_2) = \lambda (\boldsymbol{e}\cdot \boldsymbol{p}_1) < \lambda (\boldsymbol{e}\cdot \boldsymbol{p}_1) + \mu (\boldsymbol{e}\cdot \boldsymbol{p}_2) = \boldsymbol{e}\cdot \boldsymbol{p}.$$
    Therefore the upper bound of condition $3$ is also satisfied.
\item[Case 3\textbf{C} does not occur.] We claim that this case does not occur. Indeed, recall that if case $3$ occurs for $S$, we must have $S \subseteq R$. However if the boundary $B$ is in case \textbf{C} there are (at least) two full lines contained in $G$, so that $\langle G \rangle  = \mathbb R^2$, and so we must have $R = \emptyset$, giving a contradiction.
\end{description}
We are now done with the interpolation of $S$ between the rays $r_i$ and the boundary rays $B$ of the segments $S_i$ which are \textit{not} the boundary rays of $\pi$. Let us now deal with these two boundary ray interpolations.

\subsection{\texorpdfstring{Interpolating the boundary rays of $\pi$}{}}

The difference with interpolating the boundary rays of $\pi$ is that these two rays must take the same linear function values, so that after extending $s$ symmetrically into the other half plane, we obtain a continuous symmetric function. Let us denote by $r_{\pi}$ the boundary ray of $\pi$ which is on the boundary of $S_1$. 
\subsubsection{\texorpdfstring{Defining $s$ along the boundary of $\pi$ when $\dim(\langle G \rangle) \geq 1$}{}}
If we have $\dim(\langle G \rangle) \geq 1$, our choice of $\pi$ always has the boundary line of $\pi$ lying in $G$, and therefore fully contained within the boundary of $C$. We then need to have $s(\boldsymbol{p}) = - \boldsymbol{e} \cdot \boldsymbol{p}$, for all $\boldsymbol{p}$ on this boundary. We claim that this implies $s(\boldsymbol{p}) = 0$ on the boundary; indeed, if $\langle G \rangle$ is a single line $l$ then our choice of $\boldsymbol{e}$ is normal to $l$. $l$ is also the boundary line of $\pi$, so we have $\boldsymbol{e} \cdot \boldsymbol{p} = 0$ for $\boldsymbol{p} \in l$, as needed. If on the other hand we have $\langle G \rangle = \mathbb R^2$, then we must have $\boldsymbol{e} = \boldsymbol{0}$ and again the result follows. We therefore define $s(\boldsymbol{p}) = 0$ for all $\boldsymbol{p}$ on the boundary line of $\pi$. The symmetry of $s$ along the boundary of $\pi$ is clearly satisfied in this case. We must now interpolate between this boundary line and the two rays $r_1$ and $r_t$. It is enough to only check the interpolation between $r_{\pi}$ and $r_1$. Below we use $\boldsymbol{p}_1$ to denote points in $r_1$, and $\boldsymbol{p}_2$ to denote points in $r_{\pi}$. We always have $\lambda, \mu >0$ positive real numbers.

\begin{description}
    \item[$S_1$ is in case 1.] If $S_1$ is in case $1$, then we have $s(\boldsymbol{p}_1) = (\boldsymbol{v} + \boldsymbol{e}) \cdot \boldsymbol{p}_1$ for $\boldsymbol{p}_1 \in r_1$. We then get the following,
    $$s(\lambda \boldsymbol{p}_1 + \mu \boldsymbol{p}_2) = \lambda (\boldsymbol{v} + \boldsymbol{e}) \cdot \boldsymbol{p}_1 > \lambda(\boldsymbol{e} \cdot \boldsymbol{p}_1) + \underbrace{\mu(\boldsymbol{e} \cdot \boldsymbol{p}_2)}_{=0} = \boldsymbol{e} \cdot \boldsymbol{p}$$
    as needed.
    \item[$S_1$ is in case 2.] If $S_1$ is in case $2$, then we have $s(\boldsymbol{p}_1) = (\boldsymbol{v} - \boldsymbol{e}) \cdot \boldsymbol{p}_1$ for $\boldsymbol{p}_1 \in r_1$. We then have
    $$s(\lambda \boldsymbol{p}_1 + \mu \boldsymbol{p}_2) = \lambda (\boldsymbol{v} - \boldsymbol{e}) \cdot \boldsymbol{p}_1 < -\lambda(\boldsymbol{e} \cdot \boldsymbol{p}_1) - \underbrace{\mu(\boldsymbol{e} \cdot \boldsymbol{p}_2)}_{=0} = -\boldsymbol{e} \cdot \boldsymbol{p}$$ as needed.
    \item[$S_1$ is in case 3.]If $S_1$ is in case $3$, then we have $s(\boldsymbol{p}_1) = 0$ for $\boldsymbol{p}_1 \in r_1$, and any occurrence of case $3$ implies that $\boldsymbol{e} \neq \boldsymbol{0}$. Since $s$ takes value $0$ along both rays, $s$ is interpolated as $0$ everywhere in between the rays; we then clearly have $-\boldsymbol{e} \cdot \boldsymbol{p} < 0 = s(\boldsymbol{p}) \leq \boldsymbol{e} \cdot \boldsymbol{p}$, as needed for case $3$.
\end{description}
In all cases our interpolation satisfies the conditions needed, and is symmetric along the boundary of $\pi$.

\subsubsection{\texorpdfstring{Defining $s$ along the boundary of $\pi$ when $\dim(\langle G \rangle) =0$}{}}

We now consider the case where $G$ is empty. If this is the case, we claim that $\pi$ can always be chosen so that neither ray of $\pi$ lies in the boundary of $C$. 

\begin{lemma}\label{lem:boundaryGempty}
    Suppose that $\langle G \rangle  = \{\boldsymbol{0}\}$. Then there is a choice of $\pi$ such that neither boundary ray of $\pi$ lies within $\partial C$.
\end{lemma}

\begin{proof}
    Recall that if $\langle G \rangle  = \{\boldsymbol{0}\}$, then $\pi$ was simply chosen such that $R$ lies within $\pi$. If there are infinitely many such choices of $\pi$, then we are done, since some choice will avoid all boundaries of $C$. We will show that this must be the case; that is, $R$ cannot be contained in precisely one open half plane.

    For contradiction, suppose that $R$ is indeed contained in exactly one half plane $\pi$. This implies that the segments $S_1$ and $S_n$ are within $R$. This implies that the segments $-S_1$ and $-S_n$ lie outside of $C$, but then both boundary rays of $\pi$ must be in $\partial C$, contradicting that $\dim(\langle G \rangle) = 0$.
\end{proof}

Given that neither boundary ray of $\pi$ is a boundary of $C$, we can finish our interpolation argument. The relevant segments are $S_1$ and $S_n$, as these segments have $r_{\pi}$, resp. $-r_{\pi}$ on their boundary. We claim that since neither of $r_{\pi}$ and $-r_{\pi}$ are in $\partial C$, the segments $S_1$ and $S_n$ must be in the same case. Indeed, suppose that $S_1$ and $S_n$ are not in the same case. Without loss of generality, there are three options; 
\begin{itemize}
    \item If $S_1$ is in case 1 and $S_n$ is in case $2$, then $S_1 \subseteq C$ and $-S_n$ lies outside of $C$. However, $S_1$ and $-S_n$ have $r_{\pi}$ as their common border, so $r_{\pi}$ must be on the boundary of $C$, giving a contradiction.
    \item Similarly, if $S_1$ is in case $1$ and $S_n$ is in case $3$, we again find that $S_1 \subseteq C$ and $-S_n$ lies outside of $C$, yielding the same contradiction as above.
    \item If $S_1$ is in case $2$ and $S_n$ is in case $3$, then $S_n \subseteq C$ and $-S_1$ lies outside of $C$, meaning that $-r_{\pi}$ must be on the boundary of $C$, giving a contradiction.
\end{itemize}
Note that these are all the cases we have to check, up to renaming $r_{\pi}$ as $-r_{\pi}$. Therefore, we must have both $S_1$ and $S_n$ are in case $1$, or both in case $2$ (they cannot both be in case $3$, as this would mean the boundary of $\pi$ is in $\partial C$). In both cases, we define $s$ on $r_{\pi}$ (and symmetrically on $-r_{\pi}$) by following the definition given on $r_1$. For instance, if $S_1$ is in case $1$, then for $p \in r_{\pi}$ we define $s(\boldsymbol{p}) = (\boldsymbol{v} + \boldsymbol{e}) \cdot \boldsymbol{p}$, for some vector $\boldsymbol{v}$ with $\boldsymbol{v} \cdot \boldsymbol{p} >0$ for all $\boldsymbol{p} \in r_{\pi}$. In each of these cases it is easy to see that the resulting interpolation between $r_{\pi}$ and $r_1$ (and also therefore between $-r_{\pi}$ and $r_n$) satisfy the corresponding conditions within $S_1$ and $S_n$.

Since we have symmetrically defined $s$ along the boundary of $\pi$, the symmetric extension of $s$ to the whole plane is continuous. We have therefore constructed a hinge function $f(\boldsymbol{x})$ given by $s(\boldsymbol{x}) + \boldsymbol{e}\cdot \boldsymbol{x}$.

\subsection{Verifying the positivity set of the constructed hinge function}

We will now finish the proof by showing that $s(\boldsymbol{x}) + \boldsymbol{e} \cdot \boldsymbol{x}$ has $C$ as its positivity set. We begin by taking a point $\boldsymbol{p} \in R \subseteq C$. Since $\boldsymbol{p} \in R$, $\boldsymbol{e} \cdot \boldsymbol{p} >0$, and $\boldsymbol{p}$ lies in some segment $S$ which is in case $3$. Therefore $s(\boldsymbol{p})$ satisfies $-\boldsymbol{e} \cdot \boldsymbol{p}<s(\boldsymbol{p})\leq \boldsymbol{e}\cdot \boldsymbol{p}$. We then have $s(\boldsymbol{p}) + \boldsymbol{e} \cdot \boldsymbol{p} >0$, as needed.

Secondly, let us take $\boldsymbol{p} \in C \setminus R$. Since $\boldsymbol{p} \notin R$, we must have $-\boldsymbol{p} \in C$. Precisely one of $\boldsymbol{p}$ or $-\boldsymbol{p}$ lies in a segment $S$ which is in case $1$; if $\boldsymbol{p} \in S \subseteq \pi$ then we have $s(\boldsymbol{p}) = (\boldsymbol{e}+\boldsymbol{v})\cdot \boldsymbol{p} + \boldsymbol{e} \cdot \boldsymbol{p} >0$. On the other hand, if $-\boldsymbol{p} \in S$, then we have $s(\boldsymbol{p}) = s(-\boldsymbol{p}) > \boldsymbol{e}\cdot \boldsymbol{p}$, and so we have
$s(\boldsymbol{p}) + \boldsymbol{e} \cdot \boldsymbol{p} >0$. Therefore in both cases $\boldsymbol{p}$ is in the positivity set of $s(\boldsymbol{p}) + \boldsymbol{e} \cdot \boldsymbol{p}$. 

Let us now take $\boldsymbol{p} \notin C$, but such that $-\boldsymbol{p} \in C$. In this case $-\boldsymbol{p} \in R$, and so since $\boldsymbol{p}$ must lie in a segment which is in case 3, we have $-\boldsymbol{e}\cdot (-\boldsymbol{p})<s(-\boldsymbol{p})=s(\boldsymbol{p}) \leq \boldsymbol{e}\cdot (-\boldsymbol{p})$. We then have
$s(\boldsymbol{p}) + \boldsymbol{e}\cdot \boldsymbol{p} \leq \boldsymbol{e}\cdot (-\boldsymbol{p}) +\boldsymbol{e}\cdot \boldsymbol{p} =0$, so that $\boldsymbol{p}$ does not lie in the positivity set.

Our last non-boundary case is when both $\boldsymbol{p}$ and $-\boldsymbol{p}$ are not in $C$. In this case at least one of $\boldsymbol{p}$ or $-\boldsymbol{p}$ lies in a segment $S$ which is in case $2$, so that we have $s(\boldsymbol{p})=s(-\boldsymbol{p}) < -\boldsymbol{e} \cdot \boldsymbol{p}$. We then have $s(\boldsymbol{p}) + \boldsymbol{e} \cdot \boldsymbol{p} < -\boldsymbol{e} \cdot \boldsymbol{p} + \boldsymbol{e} \cdot \boldsymbol{p} = 0$, as required.

We now wish to take points from the boundary of $C$; they must all lie outside of the positivity set of $s(\boldsymbol{x}) + \boldsymbol{e} \cdot \boldsymbol{x}$. We first assume that $\langle G \rangle = \mathbb R^2$ - in this case $\boldsymbol{e}=\boldsymbol{0}$. For a point $\boldsymbol{p} \in \partial C$ with $-\boldsymbol{p} \in \partial C$, at least one of $\boldsymbol{p}$ or $-\boldsymbol{p}$ lies on a boundary ray within $\pi$ (including the boundary rays of $\pi$ themselves), and along these rays we have $s(\boldsymbol{p}) = s(-\boldsymbol{p}) = 0$, as required.

Now assume that $\langle G \rangle$ consists of a single line, which is always chosen as the boundary line of $\pi$. If we have a point $\boldsymbol{p}$ on this boundary line of $\pi$, then by our definitions we have $s(\boldsymbol{p}) =0$, as needed. If we pick a point $\boldsymbol{p} \in \partial C$ which is not on the boundary line, then we have two cases; firstly, assume that $\boldsymbol{p} \in \partial C \subseteq \pi$. In this case $\boldsymbol{p}$ lies on a boundary ray which is in case \textbf{A}, so that we have $s(\boldsymbol{p}) = -\boldsymbol{e} \cdot \boldsymbol{p}$. Then we have $s(\boldsymbol{p}) +\boldsymbol{e} \cdot \boldsymbol{p} =0$, as needed. Secondly, suppose that $\boldsymbol{p} \notin \pi$, so that $-\boldsymbol{p}$ lies on a ray which is in case \textbf{B}. We then have $s(\boldsymbol{p}) = s(-\boldsymbol{p}) = \boldsymbol{e} \cdot (-\boldsymbol{p})$, and so $s(\boldsymbol{p}) + \boldsymbol{e} \cdot \boldsymbol{p} = 0$.

Finally, we deal with the case when $\langle G \rangle = \{ 0 \}$. In this case by Lemma \ref{lem:boundaryGempty} all points $\boldsymbol{p} \in \partial C$ either themselves lie in a ray in case \textbf{A}, or otherwise $-\boldsymbol{p}$ lies in a ray which is in case \textbf{B}. By the same proofs as above, we have $s(\boldsymbol{p}) = 0$ in both of these cases.

We conclude that the positivity set of $s(\boldsymbol{x}) + \boldsymbol{e} \cdot \boldsymbol{x}$ is precisely $C$, concluding the proof of Theorem \ref{thm:main}.

\section*{Acknowledgments}
A.W. was partially supported by Austrian Science Fund project PAT2559123. We would like to thank Jose Capco for useful discussions.

\vspace{4mm}
\noindent Research Institute for Symbolic Computation (RISC)
\\Johannes Kepler University, Linz, Austria
\\\url{josef.schicho@risc.jku.at}
\\[4mm]
Johann Radon Institute for Computational and Applied Mathematics (RICAM)
\\Austrian Academy of Sciences, Linz, Austria
\\\url{ayushkumar.tewari@ricam.oeaw.ac.at}
\\\url{audie.warren@oeaw.ac.at}

\bibliography{hinges}

\begin{thebibliography}{1}

\bibitem{chua1988canonical}
{\sc L.~O. Chua and A.-C. Deng}, {\em Canonical piecewise-linear representation}, IEEE Transactions on Circuits and Systems, 35 (1988), pp.~101--111.

\bibitem{grigsby2022transversality}
{\sc J.~E. Grigsby and K.~Lindsey}, {\em On transversality of bent hyperplane arrangements and the topological expressiveness of {ReLU} neural networks}, SIAM Journal on Applied Algebra and Geometry, 6 (2022), pp.~216--242.

\bibitem{Hertrich:23}
{\sc C.~Hertrich, A.~Basu, M.~Di~Summa, and M.~Skutella}, {\em Towards lower bounds on the depth of {ReLU} neural networks}, SIAM J. Discrete Math., 37 (2023), pp.~997--1029.

\bibitem{Hornik:89}
{\sc K.~Hornik, M.~Stinchcomebe, and H.~White}, {\em Multilayer feedforward networks are universal approximators}, Neural Networks, 2 (1989), pp.~359--366.

\bibitem{Koutschan:24}
{\sc C.~Koutschan, A.~Ponomarchuk, and J.~Schicho}, {\em Representing piecewise linear functions by functions with minimal arity}, Tech. Rep. 2406.02421, ArXiV, 2024.

\bibitem{tao2022piecewise}
{\sc Q.~Tao, L.~Li, X.~Huang, X.~Xi, S.~Wang, and J.~A. Suykens}, {\em Piecewise linear neural networks and deep learning}, Nature Reviews Methods Primers, 2 (2022), p.~42.

\bibitem{Wang_Sun:05}
{\sc D.~Wang and X.~Sun}, {\em Generalization of hinging hyperplanes}, IEEE Transactions on Information Theory, 51 (2005), pp.~4425--4431.

\end{thebibliography}
\bibliographystyle{siam}

\end{document}